\documentclass{article}

\usepackage[main, final]{neurips_2026}

\usepackage[utf8]{inputenc} 
\usepackage[T1]{fontenc}    
\usepackage{hyperref}       
\usepackage{url}            
\usepackage{booktabs}       
\usepackage{amsfonts}       
\usepackage{nicefrac}       
\usepackage{microtype}      
\usepackage{xcolor}         
\usepackage{colortbl}
\usepackage{amsmath}
\usepackage{multirow}
\usepackage{graphicx}

\newtheorem{proposition}{Proposition}
\newtheorem{lemma}{Lemma}
\newtheorem{proof}{Proof}
\newtheorem{corollary}{Corollary}
\definecolor{pink}{rgb}{1, 0.8, 0.8}
\definecolor{orange}{rgb}{1, 0.9, 0.8}

\title{BudSplat: Feed-forward 3D Gaussian Splatting under a Rendering Budget}

\author{
  Muyu Xu \\
  Nanyang Technological University\\
  \texttt{muyu001@e.ntu.edu.sg} \\
  \And
  Fangneng Zhan\\
  Hong Kong University of Science and Technology \\
  \texttt{fnzhan@ust.hk} \\
  \AND
  Yu Wei\\
  Nanyang Technological University\\
  \texttt{ywei012@e.ntu.edu.sg} \\
  \And
  Ling Shao \\
  UCAS-Terminus AI Lab\\
  \texttt{ling.shao@ieee.org} \\
  \And
  Shijian Lu \\
  Nanyang Technological University\\
  \texttt{Shijian.Lu@ntu.ntu.edu.sg} \\
}

\begin{document}

\maketitle

\begin{abstract}
Feed-forward 3D Gaussian Splatting has recently emerged as a promising alternative to per-scene optimization, enabling fast novel view synthesis by predicting Gaussian primitives directly from input images. However, it quickly faces a bottleneck of rendering efficiency while scaling from sparse to dense multi-view inputs, which offer more geometric constraints but produce a rapidly growing set of candidate Gaussians. Existing solutions cannot reliably distinguish redundant primitives, leading to many redundant primitives from repeatedly observed regions and increased rendering cost without much gain in visual quality. 
We propose BudSplat, a compact feed-forward 3DGS framework that learns to keep the most valuable Gaussians under a limited primitive count. BudSplat has two novel designs. The first is a detail-aware importance score that is extracted from multi-level visual tokens, providing an explicit cue for regions where pruning is likely to harm perceptual quality. The score can be injected into voxel fusion to preserve high-value local contributors during multi-view aggregation. The second is quality-aware Gaussian pruning that selects a compact set of primitives by jointly considering rendering contribution and detail importance. The quality-aware pruning enables local quota redistribution that helps avoid over-allocating primitives to large redundant regions. Experiments on RE10K, DL3DV, and ACID show that BudSplat achieves a favorable trade-off among rendering quality and efficiency.
\end{abstract}

\section{Introduction}

Novel view synthesis aims to recover a scene representation from a set of input images and render photorealistic images from unseen viewpoints. 3D Gaussian Splatting (3DGS) has become a highly effective representation for this task due to its explicit Gaussian primitives and efficient splatting-based renderer, offering an attractive balance between reconstruction quality and rendering speed~\citep{kerbl20233d}. 
However, 3DGS relies heavily on per-scene optimization that necessitates a separate reconstruction and optimization process for each new scene, limiting its applicability to large-scale data processing, interactive applications, and practical deployment.

Feed-forward 3DGS, which predicts Gaussian primitives directly from input images in a single network pass, has recently been investigated to amortize the cost of per-scene optimization.  Different feed-forward designs have been explored, including pixel-aligned Gaussian prediction, cost-volume-based geometry reasoning, transformer-based large reconstruction models, voxel-aligned prediction, graph-based Gaussian aggregation, and unposed feed-forward reconstruction~\citep{charatan2024pixelsplat,chen2024mvsplat,zhang2024gs,ziwen2025long,wang2025volsplat,zhang2024gaussian,jiang2025anysplat}. More general feed-forward world models also begin to support 3D Gaussians as one of their output representations~\citep{liu2025worldmirror}. 
Along the above development, feed-forward 3DGS has also been moving from sparse and constrained inputs toward larger numbers of views, broader scene coverage, and weaker assumptions on camera poses.


This trend exposes an underexplored bottleneck – as the number of input views increases, the main challenge gradually shifts from reconstruction ambiguity to rendering efficiency. Denser input views provide stronger geometric constraints and better scene coverage, but these benefits come at a cost of a much larger number of candidate primitives. Typical feed-forward pipelines generate Gaussians densely from pixels, image tokens, or volumetric features, and then suppress the generated Gaussians via voxelization, pooling, or other aggregation mechanisms. Although the suppression removes some redundant primitives, the final Gaussian set can still grow rapidly with the increase of input views. Since the cost of splatting is closely related to the number of visible Gaussians, such growth increases rendering pressure and weakens one of the key advantages of 3DGS as an explicit representation.


A straightforward solution is to prune the predicted Gaussians, but simple pruning based on opacity or confidence is suboptimal as different Gaussians carry different rendering values. In multi-view reconstruction, regions that are observed many times often have high confidence, but they tend to produce highly redundant Gaussian primitives. In contrast, thin structures, occlusion boundaries, and high-frequency textures usually occupy only a small spatial extent, yet their reconstruction strongly affects perceptual quality of rendered images. If pruning is simply driven by opacity, confidence, or a global top-$k$ rule, the retained Gaussian set can be easily dominated by large, stable, and visually simple regions while those fine but important details are either averaged out during fusion or removed during pruning. This suggests that compact feed-forward 3DGS requires not only fewer Gaussians, but also reliable criteria for deciding which Gaussians are worth keeping.



We address this issue from the perspective of rendering-budget-constrained representation learning. Instead of retaining as many predicted primitives as possible, the model should preserve the subset of Gaussians that contributes most to rendering quality under a bounded primitive count. To this end, we introduce BudSplat, a compact feed-forward 3DGS framework that propagates detail-aware importance through Gaussian generation, fusion, and pruning.


BudSplat positions compactness as one part of feed-forward prediction rather than a separate post-processing step. It identifies local regions where Gaussian removal is likely to hurt perceptual quality, such as boundaries, thin geometry, and high-frequency texture. This enables importance-preserving voxel fusion that prevents valuable local contributors from being diluted by redundant observations in smooth regions, as well as quality-aware Gaussian pruning that allows selecting a compact set of primitives for rendering. On top of that, BudSplat introduces a local quota redistribution strategy that further prevents large smooth regions from monopolizing the retained primitives, ensuring sufficient representation of detail-rich regions even under a small Gaussian budget. Hence, BudSplat redefines the role of pruning in feed-forward 3DGS: pruning is not generic Gaussian reduction after dense generation but a learned allocation mechanism that decides where the limited rendered primitives should be spent. It strikes a good balance among visual fidelity, representation compactness, and rendering efficiency as the number of input views increases.



Our main contributions are summarized as follows:
\begin{itemize}
    \item We identify the primitive growth problem in feed-forward 3DGS and study Gaussian prediction from the perspective of rendering-budget-constrained representation learning.

    \item We propose a detail-aware compact 3DGS framework that integrates residual token-based importance estimation, importance-preserving voxel fusion, and quality-aware Gaussian pruning with local quota redistribution, enabling the model to retain primitives that are more valuable for rendering.

    \item Experiments on RealEstate10K, DL3DV, and ACID demonstrate that our method achieves a better trade-off between visual fidelity and rendering efficiency under challenging low-overlap evaluation settings.
\end{itemize}
\section{Related Work}

\paragraph{Feed-forward and generalizable 3D Gaussian Splatting.}
3D Gaussian Splatting (3DGS) represents scenes with explicit anisotropic Gaussian primitives and enables efficient differentiable rendering, but it still requires per-scene optimization~\citep{kerbl20233d}. Feed-forward 3DGS amortizes this cost by predicting Gaussians directly from input images. Early methods such as pixelSplat and MVSplat generate pixel-aligned Gaussians from sparse posed views, with MVSplat further using cost-volume reasoning~\citep{charatan2024pixelsplat,chen2024mvsplat}. Later works improve geometry, hierarchy, or alignment with depth priors, hierarchical Gaussians, and voxel-aligned prediction~\citep{xu2025depthsplat,liu2025monosplat,tang2024hisplat,wang2025volsplat}. Other methods scale the paradigm through transformer-based reconstruction, long-sequence modeling, iterative refinement, whole-scene reconstruction, or unposed inputs~\citep{zhang2024gs,ziwen2025long,kang2025ilrm,wang2024freesplat,wang2025freesplat++,ye2024no,xu2025freesplatter,jiang2025anysplat}. WorldMirror further studies general feed-forward world reconstruction with 3D Gaussian outputs~\citep{liu2025worldmirror}. Despite these advances, most methods still rely on dense per-view, pixel-aligned, or voxel-aligned primitive prediction, leaving the growth of rendered Gaussians under large-view inputs insufficiently addressed.

\paragraph{Compactness, fusion, and redundancy reduction.}
Many works compress optimized 3DGS scenes after reconstruction by pruning low-significance Gaussians, learning masks, reducing per-point storage, modeling inter-primitive redundancy, or combining pruning with quantization and distillation~\citep{fan2024lightgaussian,girish2024eagles,lee2024compact,liu2024compgs,ali2024trimming}. These methods demonstrate substantial redundancy in 3DGS, but mainly operate as post-reconstruction compression. In feed-forward 3DGS, FreeSplat, FreeSplat++, GGN, MuSASplat, and AnySplat reduce multi-view redundancy through feature fusion, Gaussian pooling, or differentiable voxelization~\citep{wang2024freesplat,wang2025freesplat++,zhang2024gaussian,xu2026musasplat,jiang2025anysplat}. 
These approaches reduce duplicated observations or improve cross-view aggregation, but they are not formulated as rendering-budget allocation methods and do not explicitly preserve local perceptual details during fusion. 
Our method differs from the two kinds above by treating compactness as rendering-budget allocation during feed-forward prediction, unifying detail-aware token importance, importance-preserving voxel fusion, and Gaussian-level pruning.

\section{Method}

Given a set of multi-view input images, our goal is to predict a compact 3D Gaussian representation that can be rendered efficiently while preserving feed-forward reconstruction quality. We build on an AnySplat-style pipeline, where a visual geometry backbone predicts dense 3D points and image-aligned features, and these dense candidates are fused into Gaussian primitives for rendering. This design is effective for unconstrained reconstruction, but becomes less scalable as the number of input views increases. More views provide stronger geometric constraints, yet also generate many redundant candidates from repeatedly observed smooth regions. In contrast, visually important regions such as thin structures, occlusion boundaries, and high-frequency textures often occupy only a small fraction of the scene.

Our method makes primitive generation and pruning aware of rendering value. We first estimate patch-level detail importance from multi-layer VGGT tokens, then preserve this signal during voxel fusion, and finally select a compact Gaussian set according to a quality-aware ranking criterion. The model therefore does not simply keep the most confident primitives, but allocates the retained Gaussians to regions that are more likely to affect novel-view quality.

Figure~\ref{fig:pipeline} illustrates the overall pipeline. Given multi-view RGB images, the VGGT backbone first predicts dense geometry, confidence cues, and multi-layer visual tokens. The dense geometry and image-aligned features provide an initial pool of Gaussian candidates, while the retained multi-layer tokens are passed to our detail-aware decomposition module. This module reconstructs common low-rank token components with a down-up projection structure, uses the residual as detail cues, and predicts a patch-level detail score with a lightweight detail head. Finally, the fused Gaussians are ranked by a quality-aware pruning criterion that combines opacity, detail importance, and support-based redundancy correction, followed by patch-level quota redistribution. The resulting representation keeps only a sparse set of Gaussians for rendering while preserving visually important local structures.

\begin{figure*}[t]
    \centering
    \includegraphics[width=\textwidth]{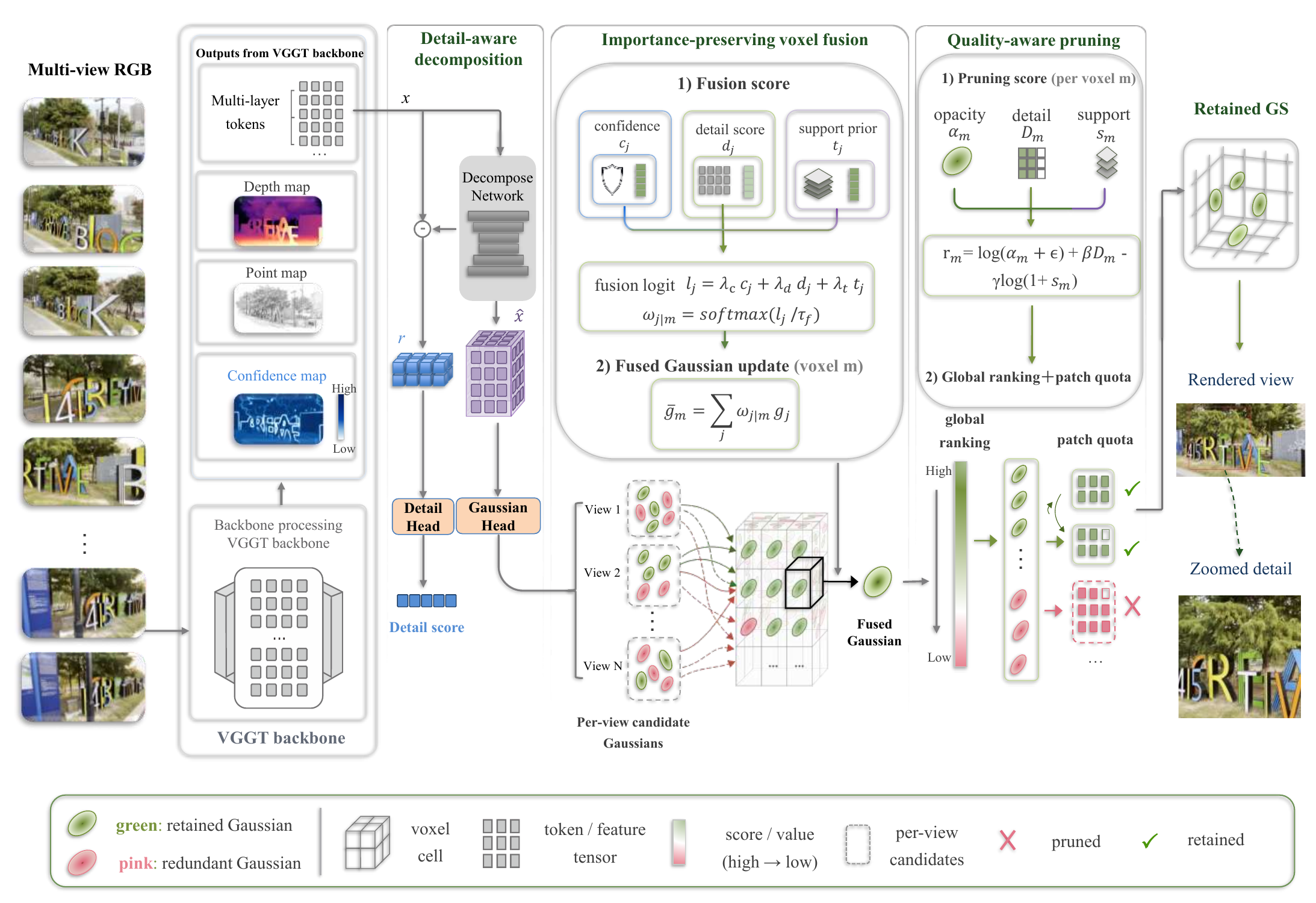}
    \caption{
    Overview of BudSplat. With dense geometry, confidence maps, and multi-layer tokens as predicted from multi-view RGB images by VGGT, the \textit{detail-aware decomposition} first extracts residual token cues and predicts patch-level detail scores. The predicted scores are then propagated to the \textit{importance-preserving voxel fusion} together with confidence and voxel support, retaining important local contributors during aggregation. Finally, the \textit{quality-aware pruning} ranks the fused Gaussians by opacity, detail importance, and support-aware redundancy correction, followed by patch-level quota redistribution to select a compact set of retained Gaussian for rendering.
    }
    \label{fig:pipeline}
\end{figure*}

\subsection{Preliminary}

\paragraph{3D Gaussian Splatting.}
3D Gaussian Splatting represents a scene as a set of anisotropic Gaussian primitives~\citep{kerbl20233d}. Each primitive is parameterized as \(g_i=(\boldsymbol{\mu}_i,\boldsymbol{\Sigma}_i,\alpha_i,\mathbf{c}_i)\), where \(\boldsymbol{\mu}_i\) is the 3D center, \(\boldsymbol{\Sigma}_i\) is the covariance, \(\alpha_i\) is the opacity, and \(\mathbf{c}_i\) is the color feature. The covariance is commonly factorized as \(\boldsymbol{\Sigma}_i=\mathbf{R}_i\operatorname{diag}(\mathbf{s}_i^2)\mathbf{R}_i^\top\), with rotation \(\mathbf{R}_i\) and scale \(\mathbf{s}_i\).

For rendering, each 3D Gaussian is projected to a 2D Gaussian. Let \(\mathbf{u}_i=\pi(\boldsymbol{\mu}_i)\) be the projected center and \(\boldsymbol{\Sigma}'_i\) be the projected covariance. Its contribution to pixel \(\mathbf{p}\) is \(a_i(\mathbf{p})=\alpha_i\exp(-\frac{1}{2}(\mathbf{p}-\mathbf{u}_i)^\top{\boldsymbol{\Sigma}'_i}^{-1}(\mathbf{p}-\mathbf{u}_i))\). After depth sorting, the rendered color is computed by alpha compositing,
\begin{equation}
\mathbf{C}(\mathbf{p})=\sum_{i\in\mathcal{N}(\mathbf{p})}T_i(\mathbf{p})a_i(\mathbf{p})\mathbf{c}_i,
\qquad
T_i(\mathbf{p})=\prod_{j<i}\left(1-a_j(\mathbf{p})\right).
\end{equation}
Although splatting is efficient, its cost remains closely related to the number of visible Gaussians. Controlling the number of rendered primitives is therefore essential for large-view-count feed-forward 3DGS.

\paragraph{VGGT.}
VGGT is a feed-forward visual geometry model that predicts 3D scene attributes from one or multiple images~\citep{wang2025vggt}. Given \(V\) context images \(\mathcal{I}=\{\mathbf{I}_v\}_{v=1}^{V}\), VGGT encodes all views with a transformer and predicts camera parameters, depth maps, point maps, and multi-view features. In our framework, dense VGGT predictions initialize candidate Gaussian positions, while retained transformer tokens provide local importance cues. We denote a retained token as \(\mathbf{x}^{\ell}_{v,p}\in\mathbb{R}^{C}\), where \(v\) indexes the view, \(p\) indexes the patch, and \(\ell\in\mathcal{L}\) indexes a retained layer.

\subsection{Detail-Aware Importance from Visual Tokens}

Dense feed-forward Gaussian prediction treats image locations almost uniformly. This is suitable for scene coverage, but inefficient when the final representation must be compact. Smooth surfaces may be observed many times and produce many confident candidates, while small structures and high-frequency details can be overwhelmed. We therefore estimate a detail-aware importance signal before voxel fusion and Gaussian pruning.

We extract this signal from retained VGGT tokens with a lightweight decomposition module. For each retained VGGT token \(\mathbf{x}^{\ell}_{v,p}\), we estimate detail importance through a low-rank adapter-style reconstruction module. The module follows a simple down-up projection structure: it first normalizes the token, projects it to a low-dimensional rank-\(r\) space, applies a non-linear activation, and projects it back to the original channel dimension. Specifically, the reconstructed common component is 
\begin{equation}
\hat{\mathbf{x}}^{\ell}_{v,p}=\mathbf{W}^{\ell}_{u}\phi(\mathbf{W}^{\ell}_{d}\operatorname{LN}(\mathbf{x}^{\ell}_{v,p})),
\end{equation}
where \(\mathbf{W}^{\ell}_{d}\in\mathbb{R}^{r\times C}\), \(\mathbf{W}^{\ell}_{u}\in\mathbb{R}^{C\times r}\), and \(r\ll C\). The residual \(\mathbf{r}^{\ell}_{v,p}=\mathbf{x}^{\ell}_{v,p}-\hat{\mathbf{x}}^{\ell}_{v,p}\) is then fed to a linear score head to obtain the per-layer detail logit 
\begin{equation}
z^{\ell}_{v,p}=\mathbf{w}_{d}^{\top}\operatorname{LN}(\mathbf{r}^{\ell}_{v,p})+b_d.
\end{equation}
The down-up projection module captures token components that can be explained by a compact low-rank subspace, while the residual branch exposes less easily reconstructed information, which can include boundaries, thin structures, and local high-frequency variations. This design avoids directly scoring the original high-dimensional token, which can be dominated by common low-dimensional components in geometry-aware transformer features. Since different transformer layers encode different levels of abstraction, we fuse detail logits from multiple retained layers using learned layer weights. The final detail logit is \(z_{v,p}=\sum_{\ell\in\mathcal{L}}\omega_{\ell}z^{\ell}_{v,p}\), where \(\omega_{\ell}=\operatorname{softmax}(\gamma)_{\ell}\), and the patch-level detail probability is \(d_{v,p}=\sigma(z_{v,p})\). We upsample this detail map to the pixel lattice and attach it to dense candidate points before voxel fusion, so the importance signal is propagated during aggregation rather than added only after Gaussian prediction.

To relate these residual features to rendering quality, we supervise the detail branch with a hybrid target. Rendering residuals from an ungated source-view reconstruction indicate where the current representation fails, while photometric and geometric variations provide a stable cue for fine structures. After normalizing these cues per view, we pool them to the patch lattice and combine them as \(y_{v,p}=\lambda_{\mathrm{res}}y^{\mathrm{res}}_{v,p}+\lambda_{\mathrm{var}}y^{\mathrm{var}}_{v,p}\). The target is stop-gradient, and the detail loss \(\mathcal{L}_{\mathrm{detail}}=\operatorname{BCEWithLogits}(z_{v,p},y_{v,p})\) encourages the branch to focus on regions that are difficult to reconstruct and rich in local structure. This supervision encourages rendering-relevant predictions from residual features, whose magnitudes alone may also reflect noise or geometry errors. Appendix~\ref{app:details} details the individual cues and their aggregation.

\subsection{Importance-Preserving Voxel Fusion}

The detail signal is useful only if it survives the transition from dense candidates to fused Gaussians. In the baseline pipeline, candidates from multiple views are grouped by spatial voxelization. This removes duplicates, but a confidence-dominated fusion rule can dilute rare local structures. For example, if a voxel contains many smooth-surface candidates and only a few boundary candidates, averaging mainly by confidence may suppress the boundary contribution.

We therefore make voxel fusion detail-aware. Let \(j\) be a dense contributor with 3D location \(\mathbf{p}_j\), feature \(\mathbf{h}_j\), confidence \(c_j\), source view \(v_j\), source patch \(p_j\), and detail prior \(d_j\) sampled from the upsampled token detail map. Contributors assigned to voxel \(m\) form \(\mathcal{V}_m=\{j\mid \operatorname{round}(\mathbf{p}_j/\delta)=m\}\), where \(\delta\) is the voxel size. We also compute a contributor-level support prior \(t_j\) from the number of candidates or source views supporting the same local voxel neighborhood, normalized within each scene. This prior provides a weak cue for multi-view consistency, while the detail score prevents repeatedly observed smooth regions from dominating the fusion.

For each contributor inside voxel \(m\), we compute 
\begin{equation}
    \ell_j=\lambda_c c_j+\lambda_d d_j+\lambda_t t_j
\end{equation}
and normalize it within the voxel as \(w_{j|m}=\operatorname{softmax}_{j\in\mathcal{V}_m}(\ell_j/\tau_f)\). The fused position, feature, and detail score are \(\bar{\mathbf{x}}_m=\sum_{j\in\mathcal{V}_m}w_{j|m}\mathbf{x}_j\), \(\bar{\mathbf{h}}_m=\sum_{j\in\mathcal{V}_m}w_{j|m}\mathbf{h}_j\), and \(\bar{d}_m=\sum_{j\in\mathcal{V}_m}w_{j|m}d_j\). This rule still uses confidence and support, but contributors with high local detail can receive larger weights even when they are not the majority inside a voxel.

After fusion, we record voxel support \(s_m=|\mathcal{V}_m|\), which counts how many dense contributors are merged into voxel \(m\). Different from the contributor-level support prior \(t_j\) used for fusion, \(s_m\) is a voxel-level statistic used later for pruning. A large \(s_m\) may indicate strong multi-view evidence, but it may also indicate redundancy, so we use it only as a mild correction in the pruning score. 

To carry the token cue into pruning, we associate each fused Gaussian with its dominant contributor \(j_m^\star=\arg\max_{j\in\mathcal{V}_m}w_{j|m}\). Its source-view/source-patch pair defines \(\operatorname{tok}(m)\), and we set \(d_{\operatorname{tok}(m)}=d_{j_m^\star}\). The final Gaussian detail score is \(D_m=(1-\eta)d_{\operatorname{tok}(m)}+\eta\bar{d}_m\), with \(\eta=0.5\), combining this dominant token cue with the detail preserved across all contributors during voxel fusion. The same patch identity is used for budget allocation below, so a voxel contributes one fused Gaussian rather than a separate primitive for each source patch.

\subsection{Quality-Aware Gaussian Pruning}

Voxel fusion reduces many duplicates, but it does not fully control rendering cost. With many input views, the fused Gaussian set can still contain many primitives from repeatedly observed regions. Rendering all of them is expensive, while naive global pruning may remove small but important details. We therefore combine a differentiable gate for learning with quality-aware ranking and quota allocation to select Gaussians under a fixed budget.

For each fused Gaussian feature \(\bar{\mathbf{h}}_m\), a lightweight gate head predicts \(a_m=f_{\mathrm{gate}}(\bar{\mathbf{h}}_m)\). The soft gate is \(q_m=f_t+(1-f_t)\sigma(a_m/\tau_g)\), where \(\tau_g\) is a temperature and \(f_t\) is a gate floor. The gate modulates opacity as \(\alpha'_m=\alpha_mq_m\), providing a differentiable training path. The gate floor prevents early collapse, and the temperature controls how sharp the selection becomes. During training, we gradually decrease the floor and anneal the keep ratio, allowing the model to first learn stable reconstruction and then adapt to stricter budgets.

For hard pruning, each fused Gaussian is ranked by
\begin{equation}
r_m=\log(\alpha_m+\epsilon)+\beta D_m-\gamma\log(1+s_m).
\label{eq:ranking}
\end{equation}
The opacity term favors primitives with strong rendering contribution, the detail term favors perceptually important local structures, and the support term mildly penalizes heavily populated voxels. This logarithmic correction does not mean that high-support Gaussians are unimportant. Instead, it prevents support alone from allowing redundant regions to monopolize the retained set. A high-support Gaussian with strong opacity or detail can still be selected.

Given a keep ratio \(\rho\) and \(M\) fused Gaussians, the target number of retained primitives is \(K=\lceil\rho M\rceil\). A global top-\(K\) rule can still favor large smooth regions because they produce many confident candidates. To reduce this bias, we redistribute the budget at the patch level using the identity inherited from each Gaussian's dominant contributor. These fixed VGGT token cells serve as allocation bins for already-fused Gaussians: cross-view merging is handled by the preceding 3D voxel fusion. We first assign each patch a base quota proportional to its candidate count, and then distribute the remaining slots according to patch detail strength and remaining capacity. The quotas sum to the same global budget \(K\), and Gaussians within each patch are selected by the same ranking score \(r_m\). This turns pruning from pure global competition into locally balanced allocation, giving detail-rich patches a better chance to retain their most valuable primitives.

This hard selection is also used in the training forward pass, while the soft gate provides the gradient path. Let \(\kappa_m\in\{0,1\}\) be the binary keep mask produced by budget selection. We use the straight-through form \(\tilde{\kappa}_m=\kappa_m+q_m-\operatorname{sg}(q_m)\), where \(\operatorname{sg}(\cdot)\) denotes stop-gradient, and render with final opacity \(\alpha^{\mathrm{final}}_m=\alpha_m q_m\tilde{\kappa}_m\). Thus, the renderer sees a compact Gaussian set in the forward pass, while the gate predictor remains trainable. At inference, the same hard selection is applied without the training schedules or straight-through gradient.

\subsection{Learning Objectives}

We train BudSplat with reconstruction losses and auxiliary objectives for detail estimation and budgeted pruning. The reconstruction term \(\mathcal{L}_{\mathrm{rec}}\) includes photometric and perceptual supervision, while the detail branch is supervised by \(\mathcal{L}_{\mathrm{detail}}\). To match the target keep ratio, we use \(\mathcal{L}_{\mathrm{budget}}=(\frac{1}{M}\sum_m q_m-\rho)^2\), together with an entropy-style regularizer to avoid premature gate collapse. We further align the gate with rendering utility using self-importance supervision, where \(u_m=|\partial\mathcal{L}_{\mathrm{rec}}/\partial\alpha_m|\) provides pseudo labels for influential Gaussians. A same-step ungated teacher regularizes the gated student in depth and normal space. The full objective is
\begin{equation}
\mathcal{L}=\mathcal{L}_{\mathrm{rec}}+\lambda_{\mathrm{detail}}\mathcal{L}_{\mathrm{detail}}+\lambda_{\mathrm{budget}}\mathcal{L}_{\mathrm{budget}}+\lambda_{\mathrm{gate}}\mathcal{L}_{\mathrm{gate}}+\lambda_{\mathrm{teach}}\mathcal{L}_{\mathrm{teach}} .
\end{equation}
\section{Experiments}

We evaluate our method on large-view-count feed-forward 3D Gaussian Splatting, focusing on both novel-view synthesis quality and rendering efficiency under limited Gaussian budgets. Our experiments compare the quality-efficiency trade-off against representative feed-forward 3DGS baselines and analyze the contribution of the proposed detail-aware fusion and pruning modules.

\subsection{Experimental Setup}

\paragraph{Datasets.}
We train our model on a mixture of RealEstate10K (RE10K)~\citep{zhou2018stereo} and DL3DV~\citep{ling2024dl3dv}, and evaluate it on RE10K, DL3DV, and ACID~\citep{liu2021infinite}. RE10K mainly contains forward-facing indoor and outdoor videos with smooth camera trajectories. DL3DV contains more diverse scenes and larger view changes, making it suitable for evaluating many-view reconstruction. ACID is used as an out-of-domain benchmark to test cross-dataset generalization. Unless otherwise specified, all methods are evaluated on the same context-target splits for each dataset.

\paragraph{Evaluation protocol.}
We build on the evaluation protocol of NoPoSplat~\citep{ye2024no} to construct a low-overlap, large-baseline stress test. Closely spaced inputs can let a compact Gaussian set overfit a narrow camera trajectory, so we evaluate on target views with large viewpoint changes and low visual overlap.
For RE10K and ACID, we start from the original NoPoSplat~\cite{ye2024no} test-view candidates and keep view pairs whose overlap is below \(30\%\). We then sample intermediate target views along the selected camera segment for novel-view evaluation. For DL3DV, we compute the view overlap between the two endpoint camera poses of each candidate trajectory segment and retain segments whose endpoint overlap is below \(30\%\). All methods are evaluated on the same context and target views. This common stress-test protocol changes target-view difficulty, so its absolute scores are not directly comparable to published standard-benchmark values.

We evaluate image quality with PSNR, SSIM, and LPIPS~\citep{zhang2018unreasonable}. In addition, we report the number of Gaussians used for rendering, denoted as \#GS, since our goal is to reduce rendering cost while preserving reconstruction quality. We evaluate all methods under different numbers of input views, including 16, 20, 24, 32, 48, and 72 views.

\paragraph{Baselines.}
We compare with representative large-view-count feed-forward 3DGS methods, including GGN~\citep{zhang2024gaussian}, AnySplat~\citep{jiang2025anysplat}, WorldMirror~\citep{liu2025worldmirror}, and VolSplat~\citep{wang2025volsplat}. GGN models relations among Gaussian groups and performs Gaussian pooling, AnySplat uses differentiable voxelization to merge dense primitives, and WorldMirror provides a general feed-forward world reconstruction framework with 3D Gaussian outputs. We also include VolSplat, which predicts voxel-aligned Gaussians with known camera poses. Since VolSplat assumes posed inputs while our method and the other baselines operate without camera poses, this comparison is not strictly fair and should be viewed as a pose-aware reference.

For each method, we use the official implementation and recommended configuration whenever available. All methods are evaluated using the same context and target views, and all rendered images are evaluated at \(448\times224\) resolution for a fair comparison. 


\begin{table}[t]
\centering
\caption{
Main quantitative comparison under increasing input views.
We report PSNR, SSIM, LPIPS, and the number of rendered Gaussians (\#GS).
For our method, Ours$_{\rho}$ denotes retaining a \(\rho\)-fraction of fused Gaussians after pruning.
Pink and orange cells indicate the best and second-best results under each dataset and view setting, following the metric direction indicated by the arrows.
OOM denotes out-of-memory during inference.
}
\label{tab:main_quantitative}
\setlength{\tabcolsep}{3pt}
\renewcommand{\arraystretch}{0.95}
\scriptsize
\resizebox{\linewidth}{!}{
\begin{tabular}{ll|cccc|cccc|cccc}
\toprule
\multirow{2}{*}{Dataset} &
\multirow{2}{*}{Method} &
\multicolumn{4}{c|}{16 Views} &
\multicolumn{4}{c|}{32 Views} &
\multicolumn{4}{c}{72 Views} \\
\cmidrule(lr){3-6}
\cmidrule(lr){7-10}
\cmidrule(lr){11-14}
&
& PSNR$\uparrow$ & SSIM$\uparrow$ & LPIPS$\downarrow$ & \#GS$\downarrow$
& PSNR$\uparrow$ & SSIM$\uparrow$ & LPIPS$\downarrow$ & \#GS$\downarrow$
& PSNR$\uparrow$ & SSIM$\uparrow$ & LPIPS$\downarrow$ & \#GS$\downarrow$ \\
\midrule

\multirow{6}{*}{RE10K}
& VolSplat (posed)    & \cellcolor{pink}23.63 & \cellcolor{pink}0.838 & 0.200 & 1605K & 22.06 & \cellcolor{orange}0.791 & 0.245 & 3210K & OOM & OOM & OOM & OOM \\
& GGN         & 19.57 & 0.738 & 0.276 & \cellcolor{orange}425K & 19.11 & 0.725 & 0.285 & 586K & OOM & OOM & OOM & OOM \\
& AnySplat    & 21.72 & 0.724 & \cellcolor{orange}0.187 & 1298K   & 22.69 & 0.756 & \cellcolor{orange}0.174 & 1363K & 23.48 & 0.789 & \cellcolor{pink}0.164 & 3461K\\
& WorldMirror & 22.37 & 0.768 & 0.194 & 1072K & 22.85 & 0.786 & 0.190 & 1782K & 22.92 & 0.791 & 0.189 & 3034K \\
\cmidrule(lr){2-14}
& Ours$_{\rho=5\%}$  & 22.48 & 0.768 & 0.193 & \cellcolor{pink}65K & \cellcolor{orange}23.08 & \cellcolor{orange}0.791 & 0.181 & \cellcolor{pink}97K & \cellcolor{orange}23.52 & \cellcolor{orange}0.803 & 0.179 & \cellcolor{pink}173K \\
& Ours$_{\rho=40\%}$ & \cellcolor{orange}23.22 & \cellcolor{orange}0.808 & \cellcolor{pink}0.186 & 519K & \cellcolor{pink}23.69 & \cellcolor{pink}0.830 & \cellcolor{pink}0.171 & \cellcolor{orange}545K & \cellcolor{pink}23.78 & \cellcolor{pink}0.842 & \cellcolor{orange}0.173 & \cellcolor{orange}1384K \\
\midrule

\multirow{6}{*}{DL3DV}
& VolSplat (posed)    & 11.71 & 0.224 & 0.674 & 1605K & 12.66 & 0.274 & 0.670 & 3210K & OOM & OOM & OOM & OOM \\
& GGN         & 12.54 & 0.318 & 0.631 & 719K & 11.91 & 0.296 & 0.628 & 1250K & OOM & OOM & OOM & OOM \\
& AnySplat    & 13.19 & 0.262 & 0.575 & 1347K & 13.87 & 0.281 & 0.573 & 2557K & 14.32 & 0.304 & 0.575 & 5176K \\
& WorldMirror & \cellcolor{pink}18.28 & \cellcolor{pink}0.568 & \cellcolor{pink}0.266 & 1345K & \cellcolor{orange}18.90 & \cellcolor{orange}0.597 & \cellcolor{pink}0.249 & 2509K & \cellcolor{orange}19.24 & \cellcolor{orange}0.614 & \cellcolor{pink}0.246 & 5002K \\
\cmidrule(lr){2-14}
& Ours$_{\rho=5\%}$  & 17.30 & 0.461 & 0.325 & \cellcolor{pink}67K & 18.08 & 0.500 & 0.302 & \cellcolor{pink}127K & 18.60 & 0.521 & 0.290 & \cellcolor{pink}258K \\
& Ours$_{\rho=40\%}$ & \cellcolor{orange}18.14 & \cellcolor{orange}0.554 & \cellcolor{orange}0.302 & \cellcolor{orange}538K & \cellcolor{pink}18.91 & \cellcolor{pink}0.602 & \cellcolor{orange}0.257 & \cellcolor{orange}1022K & \cellcolor{pink}19.27 & \cellcolor{pink}0.618 & \cellcolor{orange}0.250 & \cellcolor{orange}2070K \\
\midrule

\multirow{6}{*}{ACID}
& VolSplat (posed)    & \cellcolor{pink}25.85 & \cellcolor{pink}0.809 & \cellcolor{pink}0.199 & 1605K & \cellcolor{pink}23.76 & \cellcolor{orange}0.745 & 0.245 & 3210K & OOM & OOM & OOM & OOM \\
& GGN         & 22.39 & 0.688 & 0.298 & \cellcolor{orange}492K & 21.87 & 0.675 & 0.308 & \cellcolor{orange}671K & OOM & OOM & OOM & OOM \\
& AnySplat    & 21.95 & 0.638 & 0.240 & 1309K & 22.34 & 0.663 & 0.236 & 2310K & 22.41 & 0.674 & 0.235 & 4557K \\
& WorldMirror & 22.91 & 0.687 & 0.253 & 1245K & 23.09 & 0.723 & 0.239 & 2188K & \cellcolor{orange}22.95 & \cellcolor{orange}0.705 & 0.244 & 3989K \\
\cmidrule(lr){2-14}
& Ours$_{\rho=5\%}$  & 22.75 & 0.657 & 0.248 & \cellcolor{pink}65K & 23.00 & 0.675 & \cellcolor{orange}0.231 & \cellcolor{pink}115K & 22.68 & 0.676 & \cellcolor{orange}0.227 & \cellcolor{pink}227K \\
& Ours$_{\rho=40\%}$ & \cellcolor{orange}23.34 & \cellcolor{orange}0.747 & \cellcolor{orange}0.204 & 523K & \cellcolor{orange}23.42 & \cellcolor{pink}0.808 & \cellcolor{pink}0.200 & 924K & \cellcolor{pink}23.05 & \cellcolor{pink}0.788 & \cellcolor{pink}0.201 & \cellcolor{orange}1822K \\
\bottomrule
\end{tabular}
}
\end{table}

\paragraph{Implementation details.}
All experiments are conducted on 8 NVIDIA H100 GPUs with mixed-precision training. We initialize the model from the official VGGT pretrained weights and train on a mixture of RE10K and DL3DV with sampling weights 0.7 and 0.3, respectively. The training batch size is 1 per GPU. We use AdamW with learning rate \(2.0\times10^{-4}\), a warm-up of 180 steps, and train for 8K steps. The input resolution is \(448\times224\), and all methods are evaluated at \(448\times224\) rendering resolution.

During training, we randomly choose 4 to 24 context views. Our detail-token branch uses retained VGGT layers \(\{4,11,17,23\}\) with rank 64. We train the pruning module with stochastic keep ratios sampled from \(\{0.05,0.10,0.20,0.40\}\), biased toward lower budgets. More training details, loss weights, and module hyperparameters are provided in the supplementary material.


\subsection{Main Results}

\paragraph{Quantitative comparison.}
Table~\ref{tab:main_quantitative} compares our method with VolSplat, GGN, AnySplat, and WorldMirror on RE10K, DL3DV, and ACID. Following the low-overlap evaluation protocol described above, all methods are tested on challenging target views with large viewpoint changes. We report PSNR, SSIM, LPIPS, and the number of rendered Gaussians (\#GS), so the comparison reflects both novel-view quality and representation compactness.

We show three representative input-view settings, 16, 32, and 72 views, in the main table. Additional results for 20, 24, and 48 views are reported in the supplementary material. For our method, Ours$_{\rho=5\%}$ and Ours$_{\rho=40\%}$ denote two Gaussian budget settings, where \(\rho\) is the fraction of fused Gaussians retained after pruning. The 5\% setting evaluates aggressive compression, while the 40\% setting evaluates a higher-quality operating point.

The results show that our method achieves a strong quality-efficiency trade-off across datasets and input-view settings. VolSplat benefits from posed inputs and achieves strong results at smaller view counts on RE10K and ACID, but its Gaussian count grows rapidly and becomes memory-limited at larger view counts. GGN uses fewer Gaussians than other baselines, but its heavy grid-sampling operations become memory-intensive under large-view inputs, leading to out-of-memory cases. AnySplat and WorldMirror can handle longer inputs, but they often retain millions of Gaussians. In contrast, our method maintains competitive or better rendering quality with substantially fewer Gaussians. In particular, Ours$_{\rho=5\%}$ uses the smallest Gaussian count while preserving competitive quality, and Ours$_{\rho=40\%}$ achieves the best or second-best results in most settings.

\paragraph{Qualitative comparison.}
Figure~\ref{fig:main_qualitative} shows qualitative comparisons on challenging low-overlap target views. We compare with VolSplat, GGN, AnySplat, and WorldMirror under large-view inputs. As highlighted by the red boxes, our method produces more stable and spatially uniform renderings in detailed regions. Baselines may either retain many redundant Gaussians in easy regions or fail to allocate enough primitives to locally important structures, leading to incomplete coverage, blurred boundaries, or uneven rendering quality. In contrast, our detail-aware fusion and quality-aware pruning allocate the compact Gaussian budget to regions that are more relevant to novel-view quality, preserving thin structures, object boundaries, and high-frequency textures with fewer rendered Gaussians.

\begin{figure}[t]
\centering
\includegraphics[width=\textwidth]{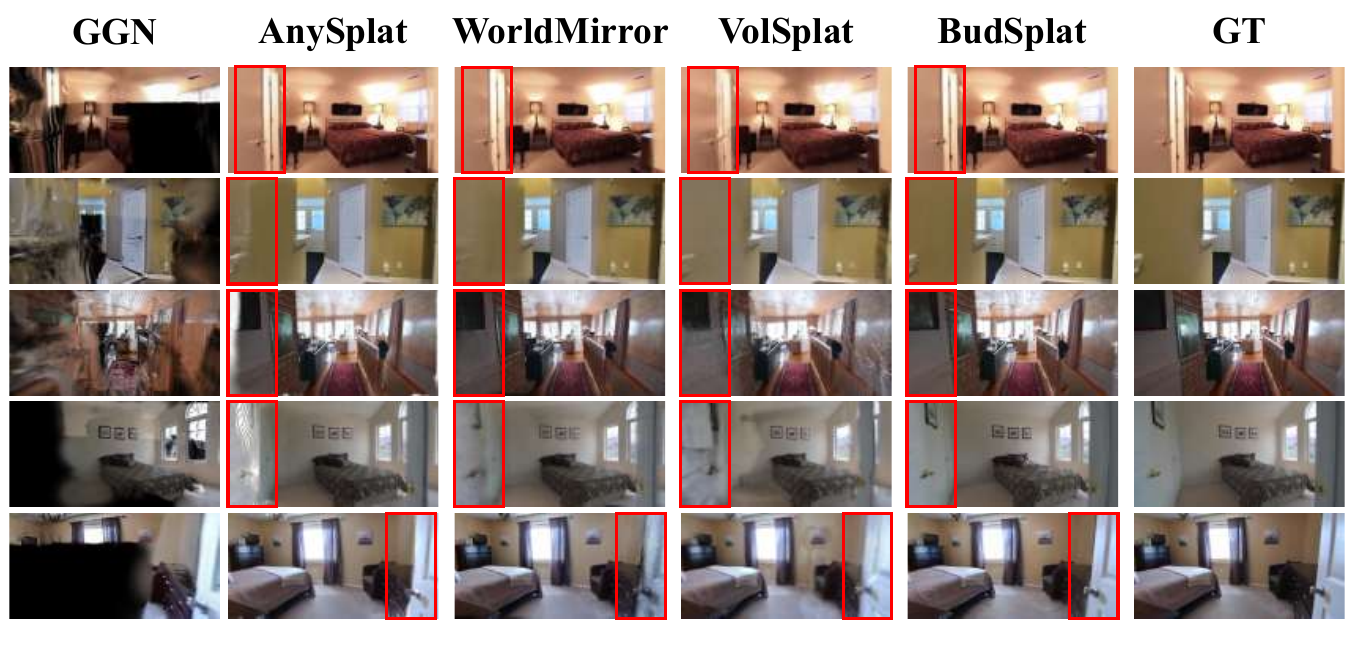}
\caption{
Qualitative comparison on challenging low-overlap target views.
Red boxes highlight regions where baselines suffer from incomplete Gaussian coverage, blurred local structures, or uneven rendering quality.
Our method allocates Gaussian budget more effectively, producing stable renderings and preserving perceptually important details with fewer Gaussians.
}
\label{fig:main_qualitative}
\end{figure}

\begin{table}[t]
\centering
\caption{
Ablation of the main design modules under the same Gaussian budget.
DT denotes detail-token importance estimation.
IVF denotes importance-preserving voxel fusion.
QAP denotes quality-aware pruning.
PQ denotes patch-level quota redistribution.
}
\label{tab:ablation_components}
\setlength{\tabcolsep}{3pt}
\renewcommand{\arraystretch}{0.95}
\scriptsize
\resizebox{\linewidth}{!}{
\begin{tabular}{lcccc|ccc}
\toprule
Variant & DT & IVF & QAP & PQ
& PSNR$\uparrow$ & SSIM$\uparrow$ & LPIPS$\downarrow$ \\
\midrule
Budgeted AnySplat-style baseline
&  &  &  &  & 13.23 & 0.522 & 0.686 \\
w/o detail-token importance
&  & \checkmark & \checkmark & \checkmark & 20.04 & 0.592 & 0.476 \\
w/o importance-preserving voxel fusion
& \checkmark &  & \checkmark & \checkmark & 22.47 & 0.768 & 0.214 \\
w/o quality-aware pruning score
& \checkmark & \checkmark &  & \checkmark & 22.93 & 0.784 & 0.209 \\
w/o patch quota redistribution
& \checkmark & \checkmark & \checkmark &  & 22.55 & 0.776 & 0.214 \\
Full model
& \checkmark & \checkmark & \checkmark & \checkmark & 23.02 & 0.791 & 0.201  \\
\bottomrule
\end{tabular}
}
\end{table}
\vspace{-0.5mm}
\subsection{Ablation Study}

We conduct an ablation study to verify the contribution of the main design modules. All variants are evaluated under the same Gaussian budget and the same low-overlap evaluation protocol of 24-view setting on RE10K, so the comparison focuses on reconstruction quality rather than differences in Gaussian count. We ablate four components: detail-token importance estimation, importance-preserving voxel fusion, quality-aware Gaussian pruning, and patch-level quota redistribution.
Table~\ref{tab:ablation_components} reports the results. The budgeted baseline uses the same Gaussian budget as our full model and selects primitives by opacity-only pruning after AnySplat-style voxel fusion. Removing the detail-token branch disables the explicit local detail prior used by both fusion and pruning. Removing importance-preserving voxel fusion falls back to confidence-dominated voxel aggregation. Removing quality-aware pruning replaces our ranking score with opacity-only selection, while removing patch quota redistribution uses global top-\(K\) selection without local budget balancing.

The results show that all proposed components contribute to the final performance. Compared with the budgeted baseline, adding the proposed modules substantially improves all metrics. Each removal from the full model leads to a performance drop, confirming that detail-aware importance, detail-preserving fusion, quality-aware ranking, and local quota redistribution are complementary. In particular, removing the detail-token branch causes the largest degradation, indicating that explicit local detail estimation is important for Gaussian prediction under a strict rendering budget.

\section{Conclusion}
We presented BudSplat, a compact feed-forward 3DGS framework for large-view-count novel view synthesis under a rendering budget. Instead of treating Gaussian reduction as post-hoc compression, BudSplat learns to allocate primitives according to rendering value. It extracts detail-aware importance from multi-level visual tokens, preserves this signal during voxel fusion, and selects fused Gaussians with quality-aware pruning and local quota redistribution. Experiments on RE10K, DL3DV, and ACID show that BudSplat achieves a favorable trade-off among rendering quality, Gaussian compactness, and efficiency. A limitation is that BudSplat reduces redundancy after the VGGT geometry backbone, but does not optimize the backbone memory itself. Thus, very dense input sequences may still be limited by VGGT computation before Gaussian pruning is applied. BudSplat is complementary to memory-efficient VGGT variants, and future work may combine backbone-level optimization with budget-aware Gaussian allocation. The method also depends on the reliability of predicted geometry and confidence cues, and extremely small budgets can still remove details in highly complex regions.

\section{Acknowledgement}
This work was supported by Ministry of Education, Singapore, under the Tier-1 project scheme (Project No. RG112/25). The project WBS is 04MNP004621C130OST01.

\newpage
\bibliographystyle{plainnat}
\bibliography{main}

\newpage
\appendix



\section{Additional implementation details}
For our method, the detail-token branch uses retained VGGT layers \(\{4,11,17,23\}\) and a rank of 64. The importance-preserving voxel fusion uses confidence, detail, and support-prior weights of 1.0, 1.5, and 0.5, respectively, with fusion temperature 1.0. The voxel-detail mixing weight is 0.5, and the support-aware ranking penalty is set to 0.15. Budget gating uses a two-layer MLP with hidden dimension 128. The gate temperature is annealed from 2.0 to 0.7, and the gate floor is annealed from 0.2 to 0. The keep ratio is annealed from 1.0 toward 0.05. In the late training stage, we enable stochastic budget training by sampling \(\rho\) from \(\{0.05,0.10,0.20,0.40\}\), with a bias toward lower budgets. LPIPS is applied after 1K steps with weight 0.05. We additionally use hard-patch LPIPS after 2.5K steps and a lightweight high-frequency loss after 1.8K steps. 

\section{Mathematical Motivation}
\label{app:theory}

This section provides a mathematical motivation for the main design choices in BudSplat. The results below are not intended to prove dataset-level improvements in PSNR or LPIPS. Instead, they justify the surrogate objectives and allocation rules used by our detail-aware fusion and quality-aware pruning modules. In particular, the analysis assumes that the learned detail scores are positively correlated with rendering-relevant local structures.

\paragraph{Notation.}
Consider dense Gaussian contributors before voxel fusion. A contributor \(j\) has 3D position \(\mathbf{x}_j\), feature \(\mathbf{h}_j\), confidence \(c_j\), detail prior \(d_j\), and support prior \(t_j\). Contributors assigned to voxel \(m\) form the set \(\mathcal{V}_m\). BudSplat computes the fusion logit \(\ell_j=\lambda_c c_j+\lambda_d d_j+\lambda_t t_j\), and the normalized fusion weight \(w_{j|m}=\operatorname{softmax}_{j\in\mathcal{V}_m}(\ell_j/\tau_f)\). The fused feature is \(\bar{\mathbf{h}}_m=\sum_{j\in\mathcal{V}_m}w_{j|m}\mathbf{h}_j\). After fusion, the resulting Gaussian has opacity \(\alpha_m\), detail score \(D_m\), and voxel support \(s_m=|\mathcal{V}_m|\).

\subsection{Residual Token Decomposition}

BudSplat estimates local detail importance from residual token components. The actual module is a learned non-linear down-up projection with detail supervision. The following idealized result explains why separating a compact common component from a residual component is a reasonable design.

\begin{proposition}[Residual energy under an idealized low-rank projection]
Let \(\mathbf{x}\in\mathbb{R}^{C}\) be a zero-mean token with covariance \(\boldsymbol{\Sigma}\). Consider an idealized rank-\(r\) orthogonal projection \(P\) that reconstructs \(\mathbf{x}\) by \(P\mathbf{x}\). The projection minimizing the reconstruction error \(\mathbb{E}\|\mathbf{x}-P\mathbf{x}\|_2^2\) is the projection onto the top-\(r\) eigenspace of \(\boldsymbol{\Sigma}\). The residual \(\mathbf{x}-P\mathbf{x}\) contains the token components outside this common low-dimensional subspace.
\end{proposition}

\begin{proof}
For an orthogonal projector \(P\), the reconstruction error is
\[
\mathbb{E}\|\mathbf{x}-P\mathbf{x}\|_2^2
=
\operatorname{tr}(\boldsymbol{\Sigma})-\operatorname{tr}(P\boldsymbol{\Sigma}).
\]
Minimizing the error is therefore equivalent to maximizing \(\operatorname{tr}(P\boldsymbol{\Sigma})\) over all rank-\(r\) orthogonal projectors. By the Rayleigh--Ritz theorem~\citep{ritz1909neue}, the maximizer is the projector onto the eigenspace spanned by the top \(r\) eigenvectors of \(\boldsymbol{\Sigma}\). Hence the residual removes the best rank-\(r\) common component under this idealized objective.
\end{proof}

\paragraph{Implication.}
Our learned down-up projection module is not restricted to be an orthogonal PCA projector, so the proposition should not be interpreted as an exact characterization of the trained network. It provides an idealized justification: reconstructing common token components with a compact low-rank module encourages the residual branch to emphasize information not explained by dominant subspace structure, which is a useful cue for boundaries, thin geometry, and high-frequency variations.

\subsection{Importance-Preserving Voxel Fusion}

The baseline confidence-only fusion uses weights proportional to \(\exp(\lambda_c c_j/\tau_f)\). BudSplat adds detail and support terms to the fusion logit. The following lemma shows how these additional terms change the relative odds of contributors inside the same voxel.

\begin{lemma}[Pairwise amplification by detail and support]
For any two contributors \(a,b\in\mathcal{V}_m\), their weight ratio under BudSplat fusion is
\[
\frac{w_{a|m}}{w_{b|m}}
=
\exp\!\left(
\frac{
\lambda_c(c_a-c_b)+\lambda_d(d_a-d_b)+\lambda_t(t_a-t_b)
}{\tau_f}
\right).
\]
Compared with confidence-only fusion, the multiplicative change in odds is
\[
\exp\!\left(
\frac{\lambda_d(d_a-d_b)+\lambda_t(t_a-t_b)}{\tau_f}
\right).
\]
Thus, if contributor \(a\) has higher combined detail and support than contributor \(b\), its relative fusion weight is increased.
\end{lemma}

\begin{proof}
The result follows directly from the definition of the softmax weights. Dividing the BudSplat ratio by the confidence-only ratio cancels the confidence term and leaves only the detail and support terms.
\end{proof}

This result formalizes the dilution issue. If a rare boundary contributor has only moderate confidence, confidence-only fusion can underweight it among many smooth-region contributors. Adding detail and support terms gives the model a mechanism to increase its relative contribution during voxel aggregation.

\begin{proposition}[Directional preservation under softmax fusion]
Let \(\bar{\mathbf{h}}_m=\sum_{j\in\mathcal{V}_m}w_{j|m}\mathbf{h}_j\) be the fused feature. For any feature direction \(\mathbf{u}\), the derivative of the fused response with respect to the logit of contributor \(k\) is
\[
\frac{\partial}{\partial \ell_k}
\left(\mathbf{u}^{\top}\bar{\mathbf{h}}_m\right)
=
\frac{1}{\tau_f}
w_{k|m}
\left(
\mathbf{u}^{\top}\mathbf{h}_k
-
\mathbf{u}^{\top}\bar{\mathbf{h}}_m
\right).
\]
Therefore, if contributor \(k\) is above the current voxel average along direction \(\mathbf{u}\), increasing its fusion logit cannot decrease the fused response along that direction.
\end{proposition}

\begin{proof}
For softmax weights, \(\partial w_{j|m}/\partial \ell_k=\tau_f^{-1}w_{j|m}(\mathbf{1}_{j=k}-w_{k|m})\). Therefore,
\[
\frac{\partial\bar{\mathbf{h}}_m}{\partial \ell_k}
=
\sum_{j\in\mathcal{V}_m}
\frac{\partial w_{j|m}}{\partial \ell_k}\mathbf{h}_j
=
\frac{1}{\tau_f}w_{k|m}
\left(
\mathbf{h}_k-\bar{\mathbf{h}}_m
\right).
\]
Taking the inner product with \(\mathbf{u}\) gives the claim.
\end{proof}

\paragraph{Implication.}
This proposition does not claim that the learned detail score is always correct. It states a conditional property: once the detail prior identifies a contributor that is important along a rendering-relevant feature direction, the fusion rule provides a monotone mechanism for preserving it.

\subsection{Quality-Aware Gaussian Pruning}

After voxel fusion, BudSplat ranks each fused Gaussian by
\[
r_m=\log(\alpha_m+\epsilon)+\beta D_m-\gamma\log(1+s_m),
\]
where \(\alpha_m\) is opacity, \(D_m\) is detail importance, and \(s_m\) is voxel support. This ranking can be viewed as the solution to a simple fixed-budget surrogate.

\begin{proposition}[Top-\(K\) optimality for a separable surrogate]
Given \(M\) fused Gaussians and a budget \(K\leq M\), consider
\[
\max_{\kappa_m\in\{0,1\}}
\sum_{m=1}^{M}\kappa_m r_m
\quad
\text{s.t.}
\quad
\sum_{m=1}^{M}\kappa_m=K .
\]
Every optimizer selects the \(K\) Gaussians with the largest ranking scores \(r_m\).
\end{proposition}

\begin{proof}
Suppose a feasible solution selects \(i\) but does not select \(j\), with \(r_j>r_i\). Replacing \(i\) by \(j\) increases the objective by \(r_j-r_i>0\), contradicting optimality. Therefore, an optimum cannot exclude a higher-scoring Gaussian while including a lower-scoring one.
\end{proof}

\begin{corollary}[Effect of support-aware redundancy correction]
If two Gaussians have identical opacity and identical detail score, but \(s_a>s_b\), then \(r_a<r_b\) whenever \(\gamma>0\).
\end{corollary}

\begin{proof}
The claim follows from the strict monotonicity of \(\log(1+s)\).
\end{proof}

\paragraph{Implication.}
The support term does not assert that high-support Gaussians are unimportant. It only prevents support alone from allowing highly populated regions to dominate the retained set. A high-support Gaussian can still be selected when it has sufficiently high opacity or detail importance.

\subsection{Patch-Level Quota Redistribution}

Global top-\(K\) selection can favor large regions because they contain many candidates. BudSplat therefore redistributes part of the budget across patch bins. The following idealized relaxation motivates proportional allocation by patch importance.

\begin{proposition}[Continuous relaxation of quota allocation]
Let \(p=1,\dots,P\) index active patches, and let \(a_p>0\) denote a patch weight derived from patch detail and available candidate count. Consider the continuous allocation problem
\[
\max_{k_p>0}
\sum_{p=1}^{P}a_p\log k_p
\quad
\text{s.t.}
\quad
\sum_{p=1}^{P}k_p=K .
\]
The unique optimum is
\[
k_p^{\star}
=
K\frac{a_p}{\sum_{q=1}^{P}a_q}.
\]
\end{proposition}

\begin{proof}
The objective is strictly concave on the positive orthant, so the KKT conditions are sufficient. The Lagrangian is
\[
\mathcal{L}(k,\mu)
=
\sum_{p=1}^{P}a_p\log k_p
-
\mu\left(\sum_{p=1}^{P}k_p-K\right).
\]
Stationarity gives \(a_p/k_p=\mu\), hence \(k_p=a_p/\mu\). Enforcing \(\sum_p k_p=K\) yields \(\mu=(\sum_q a_q)/K\), giving the stated solution.
\end{proof}

\paragraph{Implication.}
The implemented quota redistribution is discrete and includes capacity clipping and integer rounding. Therefore, it is not exactly the optimizer of the continuous problem above. The proposition only justifies the proportional-allocation principle: allocating more quota to higher-detail and higher-capacity patches is the solution of a natural concave resource allocation relaxation.

\subsection{Hybrid Detail Supervision}

The detail branch is supervised by a hybrid target that combines rendering residuals and photometric/geometric variation. The following idealized result explains why such a blend can be more stable than using either source alone.

\begin{proposition}[Variance of a hybrid detail proxy]
Let \(I\) be a latent patch-level importance variable. Suppose \(R=I+\varepsilon_r\) and \(V=I+\varepsilon_v\) are two unbiased noisy proxies, with variances \(\sigma_r^2,\sigma_v^2\) and covariance \(\sigma_{rv}\). For \(H_\lambda=\lambda R+(1-\lambda)V\), the mean-squared error is
\[
\mathbb{E}(H_\lambda-I)^2
=
\lambda^2\sigma_r^2
+
(1-\lambda)^2\sigma_v^2
+
2\lambda(1-\lambda)\sigma_{rv}.
\]
The unconstrained minimizer is
\[
\lambda^\star
=
\frac{\sigma_v^2-\sigma_{rv}}
{\sigma_r^2+\sigma_v^2-2\sigma_{rv}} .
\]
The constrained minimizer over \([0,1]\) is obtained by clipping \(\lambda^\star\) to \([0,1]\), and its error is no worse than the better individual proxy because the endpoints \(\lambda=0\) and \(\lambda=1\) are feasible.
\end{proposition}

\begin{proof}
The expression for the mean-squared error follows by expanding \(H_\lambda-I=\lambda\varepsilon_r+(1-\lambda)\varepsilon_v\). Differentiating the resulting quadratic with respect to \(\lambda\) gives the unconstrained minimizer. Since the feasible set \([0,1]\) includes both endpoints, the constrained optimum cannot be worse than selecting either \(R\) or \(V\) alone.
\end{proof}

\paragraph{Implication.}
This result relies on idealized unbiased-noise assumptions and should be interpreted as motivation rather than a guarantee. Rendering residuals provide task-specific error feedback, while variation cues provide a more stable local-structure prior. Their combination can reduce target noise when the two error sources are not perfectly correlated.

\subsection{Summary of Theoretical Support}

The results above support the design of BudSplat in four ways. First, residual token decomposition is motivated by an idealized low-rank projection view, where residual components capture information outside common token subspaces. Second, the fusion softmax provably increases the relative influence of contributors with higher detail or support priors and preserves above-average feature responses in relevant directions. Third, quality-aware pruning is the exact optimizer of a separable fixed-budget surrogate, and the support term acts as a principled redundancy correction. Fourth, patch-level quota redistribution and hybrid detail supervision are motivated by standard resource-allocation and variance-reduction arguments.

These statements do not prove that BudSplat must outperform baselines on any dataset. They show that the proposed modules are internally consistent with well-defined surrogate objectives, while the final quality-efficiency gains are established empirically.

\begin{table}[t]
\centering
\caption{
Additional main quantitative comparison 20, 24, 48 input views.
We report PSNR, SSIM, LPIPS, and the number of rendered Gaussians (\#GS).
For our method, Ours$_{\rho}$ denotes retaining a \(\rho\)-fraction of fused Gaussians after pruning.
Pink and orange cells indicate the best and second-best results under each dataset and view setting, following the metric direction indicated by the arrows.
OOM denotes out-of-memory during inference.
}
\label{tab:main_quantitative_extra}
\setlength{\tabcolsep}{3pt}
\renewcommand{\arraystretch}{0.95}
\scriptsize
\resizebox{\linewidth}{!}{
\begin{tabular}{ll|cccc|cccc|cccc}
\toprule
\multirow{2}{*}{Dataset} &
\multirow{2}{*}{Method} &
\multicolumn{4}{c|}{20 Views} &
\multicolumn{4}{c|}{24 Views} &
\multicolumn{4}{c}{48 Views} \\
\cmidrule(lr){3-6}
\cmidrule(lr){7-10}
\cmidrule(lr){11-14}
&
& PSNR$\uparrow$ & SSIM$\uparrow$ & LPIPS$\downarrow$ & \#GS$\downarrow$
& PSNR$\uparrow$ & SSIM$\uparrow$ & LPIPS$\downarrow$ & \#GS$\downarrow$
& PSNR$\uparrow$ & SSIM$\uparrow$ & LPIPS$\downarrow$ & \#GS$\downarrow$ \\
\midrule

\multirow{6}{*}{RE10K}
& VolSplat (posed)    & 23.17 & 0.824 & 0.221 & 2007K & 22.94 & 0.814 & 0.231 & 2408K & OOM & OOM & OOM & OOM \\
& GGN         & 19.43 & 0.735 & 0.278 & 444K & 19.33 & 0.732 & 0.280 & 492K & OOM & OOM & OOM & OOM \\
& AnySplat    & 22.21 & 0.737 & 0.178 & 1346K   & 22.35 & 0.745 & 0.178 & 1550K & 23.07 & 0.769 & 0.167 & 2574K\\
& WorldMirror & 22.15 & 0.740 & 0.194 & 1340K & 22.72 & 0.773 & 0.192 & 1453K & 22.93 & 0.790 & 0.187 & 2345K \\
\cmidrule(lr){2-14}
& Ours$_{\rho=5\%}$  & 22.71 & 0.775 & 0.208 & 67K & 23.02 & 0.791 & 0.201 & 78K & 23.23 & 0.797 & 0.200 & 128K \\
& Ours$_{\rho=40\%}$ & 23.29 & 0.810 & 0.185 & 538K & 23.45 & 0.816 & 0.183 & 620K & 23.74 & 0.837 & 0.171 & 1029K \\
\midrule

\multirow{6}{*}{DL3DV}
& VolSplat (posed)    & 12.35 & 0.255 & 0.665 & 2007K & 12.27 & 0.255 & 0.675 & 2408K & OOM & OOM & OOM & OOM \\
& GGN         & 12.36 & 0.313 & 0.629 & 854K & 12.14 & 0.303 & 0.628 & 1005K & OOM & OOM & OOM & OOM \\
& AnySplat    & 13.45 & 0.271 & 0.573 & 1700K & 13.63 & 0.274 & 0.571 & 2021K & 14.25 & 0.297 & 0.570 & 3792K \\
& WorldMirror & 18.39 & 0.585 & 0.289 & 1651K & 18.57 & 0.594 & 0.280 & 1949K & 19.22 & 0.614 & 0.242 & 3565K \\
\cmidrule(lr){2-14}
& Ours$_{\rho=5\%}$  & 17.47 & 0.475 & 0.320 & 85K & 17.64 & 0.482 & 0.316 & 101K & 18.43 & 0.508 & 0.296 & 189K \\
& Ours$_{\rho=40\%}$ & 18.38 & 0.571 & 0.292 & 680K & 18.66 & 0.579 & 0.281 & 808K & 19.03 & 0.604 & 0.254 & 1516K \\
\midrule

\multirow{6}{*}{ACID}
& VolSplat (posed)    & 24.57 & 0.773 & 0.224 & 2007K & 24.50 & 0.774 & 0.228 & 2408K & OOM & OOM & OOM & OOM \\
& GGN         & 22.20 & 0.683 & 0.301 & 440K & 22.08 & 0.680 & 0.304 & 488K & OOM & OOM & OOM & OOM \\
& AnySplat    & 22.17 & 0.652 & 0.237 & 1608K & 22.28 & 0.657 & 0.236 & 1876K & 22.47 & 0.672 & 0.233 & 3279K \\
& WorldMirror & 22.76 & 0.688 & 0.250 & 1500K & 22.77 & 0.692 & 0.251 & 1741K & 23.57 & 0.726 & 0.231 & 2982K \\
\cmidrule(lr){2-14}
& Ours$_{\rho=5\%}$  & 22.98 & 0.670 & 0.261 & 79K & 23.04 & 0.673 & 0.251 & 93K & 23.15 & 0.684 & 0.246 & 164K \\
& Ours$_{\rho=40\%}$ & 23.35 & 0.758 & 0.201 & 643K & 23.40 & 0.795 & 0.200 & 750K & 23.75 & 0.814 & 0.196 & 1311K \\
\bottomrule
\end{tabular}
}
\end{table}

\section{Additional main experiment results}
In this section, we provide more experiment results under our main experiment settings. Table \ref{tab:main_quantitative_extra} shows extra quantitative results of our main experiments under 20, 24, and 48 input views. For more visualization results, please refer to the anonymous link here: https://budsplat.pages.dev/

\section{Visualization of Gaussian Density Allocation}
\label{app:density_visualization}

To better understand how BudSplat allocates the Gaussian budget, we visualize both the target-view projected Gaussian density and the corresponding RGB rendering before and after pruning. Directly plotting retained and discarded Gaussians in 3D is often difficult to interpret, since large-view-count feed-forward reconstruction can produce millions of primitives. We therefore use a target-view density rendering diagnostic. For a fixed target camera, all visible Gaussians are rendered with the same small opacity and the same color, so pixels covered by more Gaussians receive larger accumulated values. This diagnostic approximates the projected Gaussian density in the target view. It is used only for analysis; normal RGB rendering still uses the original Gaussian opacity, color, covariance, and the standard rendering pipeline.

Figure~\ref{fig:density_visualization} shows the result. The top row visualizes the projected density before pruning, after pruning, the removed density, and the retention ratio. Let \(D_{\mathrm{pre}}\) and \(D_{\mathrm{post}}\) denote the density maps before and after pruning. We compute the removed density as \(\max(D_{\mathrm{pre}}-D_{\mathrm{post}},0)\), and the retention ratio as \(D_{\mathrm{post}}/(D_{\mathrm{pre}}+\epsilon)\). The removed-density map highlights where the pruned Gaussians originally contributed, while the retention-ratio map indicates which regions preserve more projected Gaussian support.

\begin{figure}[t]
    \centering
    \includegraphics[width=0.95\linewidth]{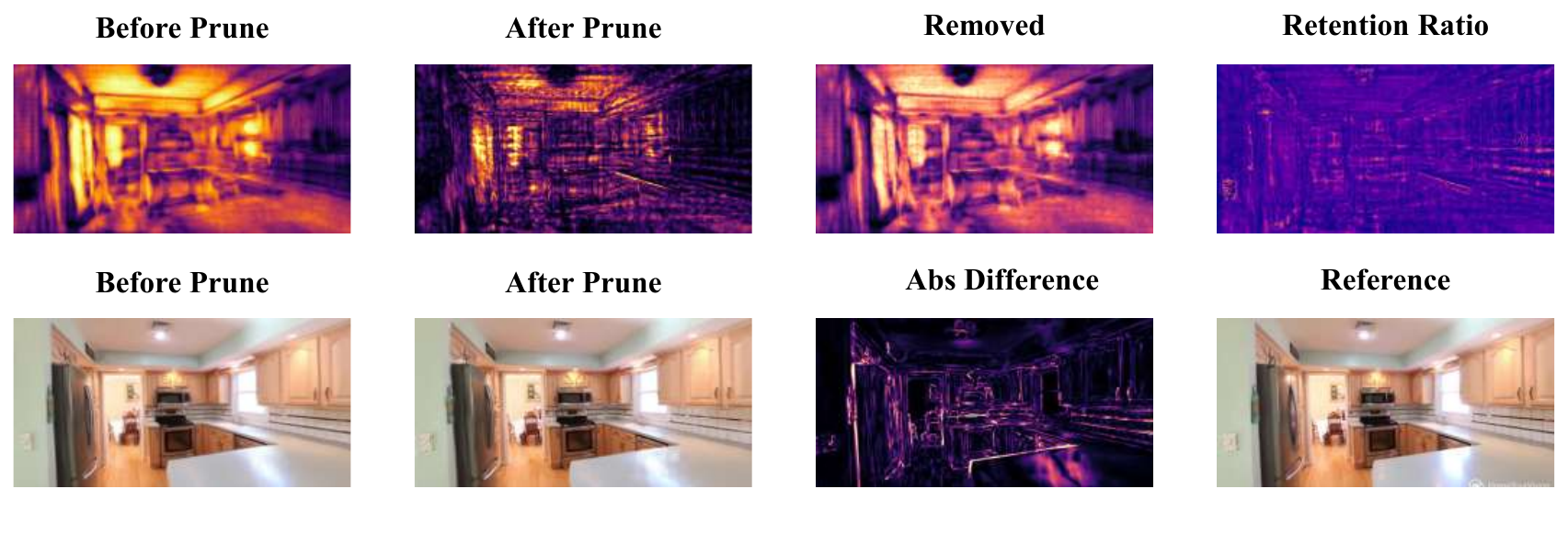}
    \caption{
    Target-view visualization of Gaussian density allocation and RGB rendering before and after pruning.
    Top row: projected Gaussian density before pruning, after pruning, removed density, and retention ratio.
    Bottom row: normal RGB rendering before pruning, after pruning, absolute difference, and reference image.
    The density diagnostic shows that pruning mainly removes redundant Gaussian support from low-frequency and highly overlapped regions, while preserving more support around high-frequency structures.
    The RGB difference map further shows that the remaining rendering changes are concentrated around edges, boundaries, and fine details rather than smooth regions.
    }
    \label{fig:density_visualization}
\end{figure}

The density maps show that BudSplat does not remove Gaussians uniformly across the image. Before pruning, broad low-frequency regions contain high projected density, which is expected because these regions are repeatedly observed by many input views and can generate many overlapping Gaussians. After pruning, the density becomes more concentrated around visually informative structures. The removed-density map indicates that many discarded Gaussians come from smooth or highly overlapped areas, where multiple primitives provide similar low-frequency information. In contrast, the retention-ratio map shows stronger preservation around edges, object boundaries, and texture-rich regions, suggesting that the compact budget is preferentially allocated to locations that are more likely to affect novel-view quality.

The bottom row provides an additional sanity check using normal RGB rendering. Unlike the density diagnostic, these RGB images are rendered with the original Gaussian attributes and the standard renderer. The before-pruning and after-pruning renderings remain visually close, indicating that a large fraction of the removed Gaussian density is redundant for final image formation. At the same time, the absolute difference map is not uniformly distributed over the image. Most differences appear around high-frequency details, such as edges, boundaries, thin structures, and texture transitions. This is expected: smooth regions can tolerate stronger Gaussian reduction because nearby primitives often provide similar color and geometry support, whereas small changes around high-frequency structures are more visible in RGB space.

This visualization is consistent with the design of BudSplat. The detail-aware token branch estimates where primitive removal is more likely to affect perceptual quality. Importance-preserving voxel fusion propagates this signal before aggregation, reducing the chance that rare but important local contributors are diluted by many redundant smooth-region observations. Quality-aware pruning and local quota redistribution then reduce unnecessary density in easy regions while retaining more support near detail-sensitive structures. Therefore, pruning behaves as budget reallocation rather than uniform Gaussian deletion: it removes many redundant primitives from low-frequency regions while preserving the primitives that are more important for rendering local details.

This analysis is qualitative and target-view dependent, but it provides an intuitive explanation for the quality-efficiency trade-off observed in the main experiments. BudSplat can substantially reduce the number of rendered Gaussians without causing large global changes in RGB rendering. The remaining differences are mainly concentrated in high-frequency regions, which are precisely the regions where compact Gaussian allocation is most challenging and where detail-aware budget control is most needed.

\section{Implementation and Analysis Details}
\label{app:details}

\paragraph{Hybrid target construction.}
For source view \(v\), we render the ungated reconstruction back to the source camera to obtain \(\hat{\mathbf I}_v\). For three-channel RGB vectors, the channel-averaged absolute residual is
\[
 e_v^{\mathrm{res}}(u)=\frac{1}{3}
 \left\|\mathbf I_v(u)-\hat{\mathbf I}_v(u)\right\|_1.
\]
The photometric and geometric variation cues are
\[
 e_v^{\mathrm{photo}}(u)=\|\nabla\mathbf I_v(u)\|_2,
 \qquad
 e_v^{\mathrm{geo}}(u)=\|\nabla z_v(u)\|_2+
 \xi\|\nabla\mathbf n_v(u)\|_2,
\]
where \(z_v\) and \(\mathbf n_v\) are the depth and normal fields, and invalid geometry is masked. Each cue is normalized per view and average-pooled to the VGGT patch lattice:
\[
 y_{v,p}^{\mathrm{res}}
 =\operatorname{Pool}_p\!\left(\operatorname{Norm}(e_v^{\mathrm{res}})\right),
 \qquad
 y_{v,p}^{\mathrm{var}}
 =\operatorname{Pool}_p\!\left(
 \frac{\operatorname{Norm}(e_v^{\mathrm{photo}})+
       \operatorname{Norm}(e_v^{\mathrm{geo}})}{2}\right).
\]
We combine these patch-level cues as \(y_{v,p}=\lambda_{\mathrm{res}}y_{v,p}^{\mathrm{res}}+\lambda_{\mathrm{var}}y_{v,p}^{\mathrm{var}}\) and stop gradients through the target before applying the BCE objective defined in the main text. The two sources of supervision play complementary roles: rendering residuals adapt to the current reconstruction, while RGB, depth, and normal variations provide structural cues even early in training. The score head thus learns to relate low-rank residual features to these cues, instead of using residual magnitude itself as the importance target.

\paragraph{Inherited and added components.}
BudSplat retains the AnySplat-style reconstruction pipeline: the pretrained VGGT geometry backbone, dense point/confidence and image-aligned feature prediction, spatial voxel grouping, Gaussian parameter decoding, and rasterizer. Our additions operate along this existing flow. The residual-token branch supplies detail scores before fusion, detail-aware weights preserve their influence inside the inherited voxels, and the learned gate, quality-aware ranking, and patch-level quotas produce the budgeted representation. Introducing importance before aggregation is essential to this design because post-hoc ranking alone cannot recover a local contributor that has already been diluted during fusion.

\paragraph{Fixed hyperparameters and scope.}
All reported datasets and view counts use \(\lambda_c=1.0\), \(\lambda_d=1.5\), \(\lambda_t=0.5\), \(\tau_f=1.0\), \(\eta=0.5\), \(\beta=1.0\), and \(\gamma=0.15\). We keep these coefficients fixed across scenes, datasets, view counts, trajectories, and keep ratios, so the reported comparisons use a common allocation rule. Within this rule, the pre-fusion contributor prior \(t_j\) supplies a weak local consistency cue, whereas the post-fusion count \(s_m=|\mathcal V_m|\) supplies a mild logarithmic correction for support already merged into one Gaussian. These stage-specific roles explain why a highly supported boundary can still be retained through its opacity and detail score, rather than being rejected solely on support. The evidence is limited to the evaluated distributions since a dedicated sparse 360-degree benchmark has not been tested. Fusion also remains dependent on the predicted geometry, since weighted averaging cannot completely separate distinct surfaces assigned to the same voxel.

\paragraph{Patch allocation and distributed training.}
After 3D voxel fusion, each Gaussian is assigned to one allocation bin through the source-view/source-patch identity of its dominant contributor. This assignment partitions an already-fused representation and it does not instantiate a separate Gaussian for every overlapping input patch. If bin \(b\) contains \(N_b\) candidates and receives integer quota \(K_b\), the allocation satisfies \(0\leq K_b\leq N_b\) and \(\sum_bK_b=K\), with the same score \(r_m\) used for selection inside each bin. The resulting keep mask enforces the budget, while the soft gate provides the training path described in the main text. During distributed training, each GPU handles one scene and its local budget before forming a scalar loss. Data parallelism then synchronizes fixed-shape network-parameter gradients, so differing Gaussian counts do not require gathering variable-length scene representations. Hard ranking and index selection are not differentiated. Gate learning uses the stated straight-through mask.

\paragraph{Training memory and efficiency scope.}
To locate the remaining memory bottleneck, we measure synchronized peak allocated memory for a complete training step and for a VGGT-only forward-backward pass with the same inputs and mixed precision. At 24 views, these peaks are approximately 42\,GB and 33\,GB per GPU, respectively; the backbone-only peak is about 79\% of the full-pipeline peak. Because the peaks are measured separately, their 9\,GB difference is not an additive estimate of BudSplat's new modules: it also includes Gaussian decoding, voxel grouping, rasterization, and backward buffers. These measurements place the timing results in context. BudSplat leaves the VGGT memory bottleneck in place, but the representation produced after this one-off computation is compact and can be reused for many novel-view renderings.


\end{document}